\documentclass{article}
\pdfoutput=0
\usepackage{amsfonts}
\usepackage{amsmath, amsthm, amssymb}
\usepackage{epsfig} 
\usepackage{subfigure}
\usepackage{graphicx} 
\usepackage{epsfig} 
\usepackage{subfigure}
\usepackage[numbers]{natbib}
\usepackage{geometry}
\usepackage{algorithm}
\usepackage{algorithmic}
\usepackage{booktabs}
\usepackage{amsthm}
\usepackage{amsfonts}
\usepackage{amssymb}
\usepackage{amsmath}
\title{A Theoretical Analysis of NDCG Type Ranking Measures}

\author{Yining Wang ({antoniowyn@gmail.com}) \\
        Institute for Interdisciplinary Information Sciences, Tsinghua University,\\
        Beijing, P.R.China
        \and
        Liwei Wang ({wanglw@cis.pku.edu.cn}) \\
        School of Electronics Engineering and Computer Science,\\
        Peking University \\
        Beijing, P.R.China
        \and
        Yuanzhi Li ({invinciblec.lee@gmail.com}) \\
        Institute for Interdisciplinary Information Sciences, Tsinghua University,\\
        Beijing, P.R.China
        \and
        Di He ({wolfink@gmail.com}) \\
        School of Electronics Engineering and Computer Science,\\
        Peking University \\
        Beijing, P.R.China
        \and
        Tie-Yan Liu ({Tie-Yan.Liu@microsoft.com})\\
        Microsoft Research Asia,\\
        Beijing, P.R. China
        \and
        Wei Chen ({wche@microsoft.com}) \\
        Microsoft Research Asia,\\
        Beijing, P.R. China
}

\begin{document}
\maketitle
\newcommand{\fix}{\marginpar{FIX}}
\newcommand{\new}{\marginpar{NEW}}


\newtheorem{definition}{Definition}
\newtheorem{theorem}{Theorem}
\newtheorem{proposition}{Proposition}
\newtheorem{corollary}{Corollary}
\newtheorem{conjecture}{Conjecture}
\newtheorem{remark}{Remark}
\newtheorem{claim}{Claim}
\newtheorem{lemma}{Lemma}

\newcommand{\ud}{\mathrm{d}}
\newcommand{\li}{\mathrm{Li}}
\newcommand{\sli}{\mathrm{li}}
\newcommand{\Span}{\mathrm{span}}
\newcommand{\mc}{\mathcal}

\maketitle

\begin{abstract} A central
problem in ranking is to design a ranking measure for evaluation of
ranking functions. In this paper we study, from a theoretical perspective,
the widely used Normalized Discounted Cumulative Gain (NDCG)-type ranking measures. Although there are extensive empirical studies of NDCG, little is known about its theoretical properties. We first show that, whatever the ranking function is, the
standard NDCG which adopts a logarithmic discount, converges to $1$
as the number of items to rank goes to infinity. On the first sight,
this result is very surprising. It seems to imply that NDCG cannot differentiate good and bad
ranking functions, contradicting to the empirical success of NDCG in
many applications. In order to have a deeper understanding of ranking measures in general, we propose a notion referred to as \emph{consistent distinguishability}. This notion captures the intuition that a ranking measure should have such a property: For every pair of substantially different ranking functions, the ranking measure can decide which one is better in a \emph{consistent} manner on almost all datasets. We show that NDCG with logarithmic discount has consistent distinguishability although it converges to the same limit for all ranking functions. We next characterize the set of all feasible discount functions for NDCG according to the concept of consistent distinguishability. Specifically we show that whether NDCG has consistent distinguishability depends on how fast the discount decays, and $r^{-1}$ is a critical point. We then turn to the cut-off version of NDCG,
i.e., NDCG@k. We analyze the distinguishability of NDCG@k for various choices of k and the discount functions. Experimental results on real Web search datasets agree well with the theory.\\
\end{abstract}

\section{Introduction}\label{section:Introduction}

Ranking has been extensively studied in information retrieval,
machine learning and statistics. It plays a central role in various
applications such as search engine, recommendation system, expert
finding, to name a few. In many situations one wants to have, by
learning, a good ranking function \citep{crammer2001pranking,
rankboost, Joachims02}. Thus a fundamental problem is how to design
a ranking measure to evaluate the performance of a ranking function.

Unlike classification and regression for which there are simple and
natural performance measures, evaluating ranking functions has
proved to be more difficult. Suppose there are $n$ objects to rank.
A ranking evaluation measure must induce a total order on the $n!$
possible ranking results. There seem to be many ways to define
ranking measures and several evaluation measures have been proposed
\citep{chapelle2009expected, turpin2006user, baeza1999modern,
agarwal2004generalization, rudin2009p}. In fact, as pointed out by
some authors, there is no single optimal ranking measure that works
for any application \citep{croft2010search}.

The focus of this work is the Normalized Discounted Cumulative Gain
(NDCG) which is one of the most popular evaluation measures in Web
search \citep{NDCG1, NDCG2}. NDCG has two advantages compared to
many other measures. First, NDCG allows each retrieved document has
graded relevance while most traditional ranking measures only allow
binary relevance. That is, each document is viewed as either
relevant or not relevant by previous ranking measures, while there
can be degrees of relevancy for documents in NDCG. Second, NDCG
involves a discount function over the rank while many other measures
uniformly weight all positions. This feature is particularly
important for search engines as users care top ranked documents much
more than others.

The importance of NDCG as well as other ranking measures in modern
search engines is not limited as evaluation metrics. Currently
ranking measures are also used as guidance for design of ranking
functions due to works from the learning to rank area. Although
early results of learning to rank often reduce ranking problem to
classification or regression \citep{crammer2001pranking, rankboost,
Joachims02, nallapati2004discriminative, balcan2008robust}, recently
there is evidence that learning a ranking function by optimizing a
ranking measure such as NDCG is a promising approach
\citep{valizadegan2009learning, yue2007support}. However, using the
ranking measure as objective function to optimize is computationally
intractable. Inspired by approaches in classification, some state of
the art algorithms optimize a surrogate loss instead
\citep{quoc2007learning, xia2008listwise}.

In the past a few years, there is rapidly growing interest in
studying consistency of learning to rank algorithms that optimize
surrogate losses. Such studies are motivated by the research of
consistency of surrogate losses for classification
\citep{zhang2004statistical1, bartlett2006convexity,
zhang2004statistical2, tewari2007consistency}, which is a
well-established theory in machine learning. Consistency of ranking
is more complicated than classification as there are more than one
possible ranking measures. One needs to study consistency with
respect to a specific ranking measure. That is, whether the
minimization of the surrogate leads to optimal predictions according
to the risk defined by the given evaluation measure.

The research of consistency for ranking was initiated in
\citep{cossock2008statistical, duchi2010consistency}. In fact,
\cite{duchi2010consistency} showed that no convex surrogate loss can
be consistent with the Pairwise Disagreement (PD) measure. This
result was further generalized in \citep{buffoni2011learning,
Clacutement_NIPS_12}, where non-existence of convex surrogate loss
with Average Precision and Expected Reciprocal Rank were proved.

In contrast to the above negative results, \cite{ravikumar2011ndcg}
showed that there do exist NDCG consistent surrogates. Furthermore,
by using a slightly stronger notion of NDCG consistency they showed
that any NDCG consistent surrogate must be a Bregman distance. In a
sense, these results mean that NDCG is a good ranking measure from a
learning-to-rank point of view.

NDCG is a normalization of the Discounted Cumulative Gain (DCG)
measure. (For formal definition of both DCG and NDCG, please see
Section \ref{section:Preliminaries}.) DCG is a weighted sum of the
degree of relevancy of the ranked items. The weight is a decreasing
function of the rank (position) of the object, and therefore called
discount. The original reason for introducing the discount is that
the probability that a user views a document decreases with respect
to its rank. NDCG normalizes DCG by the Ideal DCG (IDCG), which is
simply the DCG measure of the best ranking result. Thus NDCG measure
is always a number in $[0,1]$. Strictly speaking, NDCG is a family
of ranking measures, since there is flexibility in choosing the
discount function. The logarithmic discount $\frac{1}{\log(1+r)}$,
where $r$ is the rank, dominated the literature and applications. We
will refer to NDCG with logarithmic discount as the \emph{standard}
NDCG. Another discount function appeared in literature is $r^{-1}$,
which is called Zipfian in Information Retrieval \citep{CIKM09}.
Search engine systems also use a cut-off top-k version of NDCG. That
is, the discount is set to be zero for ranks larger than $k$. Such
NDCG measure is usually referred to as NDCG@k.

Given the importance and popularity of NDCG, there have been
extensive studies on this measure, mainly in the field of
Information Retrieval \citep{al2007relationship, CIKM09, aslam2005maximum, voorhees2001evaluation, sakai2006evaluating}. All these research are conducted from an empirical perspective by doing experiments on benchmark datasets.
Although these works gained insights
about NDCG, there are still important issues unaddressed. We list a
few questions that naturally arise.

\begin{itemize}
\item As pointed out in \citep{croft2010search}, there has not been
any theoretically sound justification for using a logarithmic
($\frac{1}{\log(1+r)}$) discount other than the fact that it is a
smooth decay.

\item Is it possible to characterize the class of discount functions
that are feasible for NDCG?

\item For the standard NDCG@k, the discount is a combination of a
very slow logarithmic decay and a hard cut-off. Why don't simply use
a smooth discount that decays fast?
\end{itemize}

In this paper, we study the NDCG type ranking measures and address
the above questions from a theoretical perspective. The goal of our
study is twofold. First, we aim to provide a better understanding
and theoretical justification of NDCG as an evaluation measure.
Second, we hope that our results would shed light and be useful for
further research on learning to rank based on NDCG. Specifically we
analyze the behavior of NDCG as the number of objects to rank
getting large. Asymptotics, including convergence and asymptotic
normality, of many traditional ranking measures have been studied in
depth in statistics, especially for Linear Rank Statistics and
measures that are U-statistics \citep{hajek1967theory,
kendall1938new}. \cite{Clemencon_NIPS_08} observed that ranking
measures such as Area under the ROC Curve (AUC), P-Norm Push and DCG
can be viewed as Conditional Linear Rank Statistics. That is,
conditioned on the relevance degrees of the items, these measures
are Linear Rank Statistics \citep{hajek1967theory}. They show
uniform convergence based on an orthogonal decomposition of the
measure. The convergence relies on the fact that the measure can be
represented as a (conditional) average of a fixed score-generating
function. Part of our work consider the convergence of NDCG and are
closely related to \citep{Clemencon_NIPS_08}. However, their results
do not apply to our problem, because the score-generating function
for NDCG is not fixed, it changes with the number of objects.

\subsection{Our Results}

Our study starts from an analysis of the standard NDCG (i.e., the
one using logarithmic discount). The first discovery is that for
\emph{every} ranking function, the NDCG measure converges to $1$ as
the number of items to rank goes to infinity. This result is surprising. On the first sight
it seems to mean that the widely used
standard NDCG cannot differentiate good and bad ranking systems when
the data is of large size. This problem may be serious because huge dataset is common in
applications such as Web search.

To have a deeper understanding of NDCG, we first study what are the desired properties a good ranking measure should have. In this paper we propose a notion referred to as \emph{consistent distinguishability}, which we believe that every ranking measure needs to have. Before describing the definition of consistent distinguishability, let us see a motivating example. Suppose we want to select, from two ranking functions $f_1, f_2$, a better one on ranking ``sea'' images (that is, if an image contains sea, we hope it is ranked near the top). Since there are billions of sea images on the web, a commonly used method is to randomly draw, say, a million data and evaluate the two functions on them. A crucial assumption underlying this approach is that the evaluation result will be ``stable'' on large datasets. That is, if on this randomly drawn dataset $f_1$ is better than $f_2$ according to the ranking measure, then with high probability over the random draw of another large dataset, $f_1$ should still be better than $f_2$. In other words, $f_1$ is \emph{consistently} better than $f_2$ according to the ranking measure.

Our definition of consistent distinguishability captures the above intuition. It requires that for two substantially different ranking functions, the ranking measure can decide which one is better consistently on almost all datasets. (See Definition \ref{definition:comparable_whp} for formal description.) In a broader sense, consistent distinguishability is a desired property to all performance statistics (not only to ranking). For classification and regression, this property trivially holds because of the simplicity of the evaluation measures. For ranking however, things are much more complicated. It is not a priori clear whether important ranking measures such as NDCG have consistent distinguishability.


Our next main result shows that although the standard NDCG always
converges to $1$, it can consistently distinguishes every pair of substantially
different ranking functions. Therefore, if one ignores the numerical scaling problem, standard NDCG is a good ranking measure.

We then study NDCG with other possible discount. We characterize the class of discount functions that are feasible for
NDCG. It turns out that the Zipfian $r^{-1}$ is a critical point. If
a discount function decays slower than $r^{-1}$, the resulting NDCG
measure has strong power of consistent distinguishability. If a
discount decays substantially faster than $r^{-1}$, then it does not
have this desired property. Even more, such ranking measures do
not converge as the number of objects to rank goes to infinity.

Interestingly, this characterization result also provides a better understanding of
the cut-off version NDCG@k. In particular, it gives a theoretical
explanation to the previous question that why popular NDCG@k uses a
combination of slow logarithmic decay and a hard cut-off as its
discount rather than a smooth discount which decays fast.

Finally we consider how to choose the cut-off threshold for NDCG@k
from the distinguishability point of view. We analyze the behavior
of the measure for various choices of $k$ as well as the discount.
We suggest that choosing $k$ as certain function of the size of the
dataset may be appropriate.

The rest of this paper is organized as follows. Section
\ref{section:Preliminaries} provides basic notions and definitions.
Section \ref{section:Main_Results} contains the main theorems and
key lemmas for the distinguishability theorem. The experimental results are given in \ref{section:experiments}. All proofs are given in Appendix
\ref{Section:key_lemmas}-\ref{section:proof_comparable_r^beta}.

\begin{section}{Preliminaries} \label{section:Preliminaries}

Let $\mathcal{X}$ be the instance space, and let $x_1,\ldots,x_n$
($x_i \in \mathcal{X}$) be $n$ objects to rank. Let $\mathcal{Y}$ be
a finite set of degrees of relevancy. The simplest case is
$\mathcal{Y}= \{0,1\}$, where $0$ corresponds to ``irrelevant'' and
$1$ corresponds to ``relevant''. Generally $\mathcal{Y}$ may contain
more numbers; and for $y \in \mathcal{Y}$, the larger $y$ is, the
more relevant it represents. Let $f$ be a ranking
function\footnote{The ranking function we defined is often called
scoring function in literature; and ranking function has a more
general definition: For fixed $n$, a general ranking function can be
any permutation on $[n]$. However, scoring functions are used by
most search engines. Also in this paper we study the behavior of the
ranking measure of a fixed ranking function as $n$ grows, so we
focus on scoring functions. But note that Theorem
\ref{theorem:converge_2_1} and Theorem \ref{prop_unbound} hold for
any sequence of general ranking functions.}. We assume that $f$ is a
mapping from $\mathcal{X}$ to $\mathbb{R}$. For each object $x \in
\mathcal{X}$, $f$ gives it a score $f(x)$. For $n$ objects
$x_1,\ldots,x_n$, $f$ ranks them according to their scores
$f(x_1),\ldots,f(x_n)$. The resulting ranking list, denoted by
$x^f_{(1)},\ldots,x^f_{(n)}$, satisfies $f\left(x^f_{(1)}\right)\ge
\ldots \ge f\left(x^f_{(n)}\right)$.


Let $y_1,\ldots,y_n$ ($y_i \in \mathcal{Y}$) be the degree of
relevancy associated with $x_1,\ldots,x_n$. We will denote by
$S_n=\{(x_1,y_1),\ldots,(x_n,y_n)\}$ the set of data to rank. As in
existing literature \citep{rankboost, clemenccon2008ranking}, we
assume that $(x_1,y_1),\ldots,(x_n,y_n)$ are i.i.d. sample drawn
from an underlying distribution $P_{XY}$ over $\mathcal{X} \times
\mathcal{Y}$. Also let $y^f_{(1)},\ldots,y^f_{(n)}$ be the
corresponding relevancy of $x^f_{(1)},\ldots,x^f_{(n)}$.

The following is the formal definition of NDCG. Here we give a
slightly simplified version tailored to our problem.

\begin{definition}
Let $D(r)$ ($r \ge 1$) be a discount function. Let $f$ be a ranking
function, and $S_n$ be a dataset. The Discounted Cumulative Gain
(DCG) of $f$ on $S_n$ with discount $D$ is defined
as\footnote{Usually DCG is defined as $\mathrm{DCG}_D(f,S_n) =
\sum_{r=1}^{n} G(y^f_{(r)}) D(r)$, where $G$ is a monotone
increasing function (e.g., $G(y)= 2^y -1$). Here we omit $G$ for
notational simplicity. This does not lose any generality as we can
assume that $\mc{Y}$ changes to $G(\mc{Y})$.}
\begin{equation}\label{definition:DCG}{
\mathrm{DCG}_D(f,S_n) =\sum_{r=1}^{n} y^f_{(r)} D(r).}
\end{equation}
Let the Ideal DCG defined as $\mathrm{IDCG}_D(S_n) =
\max_{f'}\sum_{r=1}^{n} y^{f'}_{(r)} D(r)$ be the DCG value of the
best ranking function on $S_n$.


The NDCG of $f$ on $S_n$ with discount $D$ is defined as
\begin{equation}\label{definition:NDCG}{
\mathrm{NDCG}_D(f,S_n)
=\frac{\mathrm{DCG}_D(f,S_n)}{\mathrm{IDCG}_D(S_n)}.}
\end{equation}
\end{definition}

We call NDCG \emph{standard}, if its associated discount function is
the inverse logarithm decay $D(r) = \frac{1}{\log(1+r)}$.
Note that the base of the logarithm does not matter for NDCG, since
constant scaling will cancel out due to normalization. We will
assume it is the natural logarithm throughout this paper.

An important property of eq.(\ref{definition:NDCG}) is that if a
ranking function $f'$ preserves the order of the ranking function
$f$, then $\mathrm{NDCG}_D(f',S_n) = \mathrm{NDCG}_D(f,S_n)$ for all
$S_n$. Here by preserving order we mean that for $\forall x,x' \in
\mathcal{X}$, $f(x)>f(x')$ implies $f'(x)>f'(x')$, and vice versa.
Thus the ranking measure NDCG is not just defined on a single
function $f$, but indeed defined on an equivalent class of ranking
functions which preserve order of each other.

Below we will frequently use a special ranking function $\tilde{f}$
that preserves the order of $f$.

\begin{definition}
Let $f$ be a ranking function. We call $\tilde{f}$ the canonical
version of $f$, which is defined as
\begin{equation*}\small{
\tilde{f}(x) = \Pr_{X \sim P_X}[f(X) \le f(x)].}
\end{equation*}
\end{definition}

The canonical $\tilde{f}$ has the following properties, which can be
easily proved by the definition.

\begin{lemma}
For every ranking function $f$, its canonical version $\tilde{f}$
preserves the order of $f$. In addition, $\tilde{f}(X)$ has uniform
distribution on $[0,1]$.
\end{lemma}

Finally, we point out that although originally the discount $D(r)$
is defined on positive integers $r$, below we will often treat
$D(r)$ as a function of a real variable. That is, we view $r$ take
nonnegative real values. We will also consider derivative and
integral of $D(r)$, denoted by $D'(r)$ and $\int D(r) \ud r$
respectively.

\end{section}

\section{Main Results}\label{section:Main_Results}

In this section, we give the main results of the paper. In Section
\ref{subsection:standard_NDCG} we study the standard NDCG, i.e.,
NDCG with logarithmic discount. In Section
\ref{subsection:feasible_discount} we consider feasible discount
other than the standard logarithmic one. We analyze the top-k
cut-off version NDCG@k in Section \ref{subsection_NDCG@k}. For
clarity reasons, some of the results in Section
\ref{subsection:standard_NDCG}, \ref{subsection:feasible_discount},
and \ref{subsection_NDCG@k} are given for the simplest case that the
relevance score is binary. Section \ref{subsection:General_Cases}
provides complete results for the general case.

\subsection{Standard NDCG}\label{subsection:standard_NDCG}

To study the behavior of the standard NDCG, we first consider the
limit of this measure when the number of objects to rank goes to
infinity. As stated in Section \ref{section:Preliminaries}, we
assume the data are i.i.d. drawn from some fixed underlying
distribution. Surprisingly, it is easy to show that for every
ranking function, standard NDCG converges to $1$ almost surely.

\begin{theorem} \label{theorem:converge_2_1}
Let $D(r) = \frac{1}{\log (1+r)}$. Then for every ranking function $f$,
\begin{equation*}\small{
\mathrm{NDCG}_D(f,S_n) \rightarrow 1,~~~~~~~a.s.}
\end{equation*}
\end{theorem}

The proof is given in Appendix
\ref{section:proof_feasible_discount}.\\

At the first glance, the above result is quite negative for standard
NDCG. It seems to say that in the limiting case, standard NDCG
cannot differentiate ranking functions. However, Theorem
\ref{theorem:converge_2_1} only considers the limits. To have a
better understanding of NDCG, we need to make a deeper analysis of
its power of distinguishability. In particular, Theorem \ref{theorem:converge_2_1} does not rule out the possibility that the standard NDCG can consistently distinguish substantially different ranking functions. Below we give the formal definition that two ranking functions are consistently distinguishable by a ranking measure $\mc{M}$.





\begin{definition}\label{definition:comparable_whp}
Let $(x_1,y_1),(x_2,y_2),\ldots$ be i.i.d. instance-label pairs
drawn from the underlying distribution $P_{XY}$ over $\mc{X} \times
\mc{Y}$. Let $S_n=\{(x_1,y_1),\ldots,(x_n,y_n)\}$. A pair of ranking
functions $f_0$, $f_1$ is said to be \emph{consistently distinguishable} by a ranking measure $\mc{M}$, if there exists
a negligible function\footnote{A negligible function
$\mathrm{neg}(N)$ means that for $\forall c, \mathrm{neg}(N)<
N^{-c}$ for sufficiently large $N$.} $\mathrm{neg}(N)$ and $b \in
\{0,1\}$ such that for every sufficiently large $N$, with
probability $1-\mathrm{neg}(N)$,
\begin{equation*}\small{
\mc{M}(f_b,S_{n}) > \mc{M}(f_{1-b},S_{n}) ,}
\end{equation*}
holds for all $n \ge N$ simultaneously.
\end{definition}

Consistent distinguishability is appealing. One would like a
ranking measure $\mc{M}$ to have the property that every two
substantially different ranking functions are consistently
distinguishable by $\mc{M}$. The next theorem shows that standard NDCG does have such a desired
property. For clarity, here we state the
theorem for the simple binary relevance case, i.e., $\mc{Y} =
\{0,1\}$. It is easy to extend the result to the general case that
$\mc{Y}$ is any finite set.

\begin{theorem}\label{theorem:comparable}
For every pair of ranking functions
$f_0, f_1$, let $\overline y^{f_i}(s) = \Pr[Y=1|\tilde{f}_i(X)=s]$,
$i=0,1$. Assume $\overline y^{f_0}(s)$ and $\overline y^{f_1}(s)$
are H\"{o}lder continuous in $s$. Then, unless $\overline y^{f_0}(s)
= \overline y^{f_1}(s)$ almost everywhere on $[0,1]$, $f_0$ and
$f_1$ are consistently distinguishable by standard
NDCG.
\end{theorem}

The proof is given in Appendix \ref{Section:key_lemmas}.\\

Theorem \ref{theorem:comparable} provides theoretical justification
for using standard NDCG as a ranking measure, and answers the first question raised in Introduction. Although standard NDCG converges
to the same limit for all ranking functions, it is still a good ranking measure with strong consistent distinguishability (if we ignore the numerical scaling issue).

\subsection{Characterization of Feasible Discount
Functions}\label{subsection:feasible_discount}

In the previous section we demonstrate that standard NDCG is a good
ranking measure. In both literatures and real applications, standard
NDCG is dominant. However, there is no known theoretical evidence
that the logarithmic function is the only feasible discount, or it
is the optimal one. In this subsection, we will investigate other
discount functions. We study the asymptotic behavior and
distinguishability of the induced  NDCG measures and compare to the
standard NDCG. Finally, we will characterize the class of discount
functions which we think are feasible for NDCG. For the sake of
clarity, the results in this subsection are given for the simplest
case that $\mc{Y} = \{0,1\}$. Complete results will be given in
Section \ref{subsection:General_Cases}.

Standard NDCG utilizes the logarithmic discount which decays slowly.
In the following we first consider a discount that decays a little
faster. Specifically we consider $D(r) = r^{-\beta}$ ($0< \beta <
1$). 
Let us first investigate the limit of the ranking measure as the
number of objects goes to infinity.

\begin{theorem}\label{theorem:r^-alpha}
Assume $D(r) = r^{-\beta}$ where $\beta\in(0,1)$. Assume also $p=
\Pr[Y=1] > 0$ and $\overline{y}^f(s) = \Pr[Y=1|\tilde{f}(X)=s]$ is a
continuous function. Then
\begin{equation}{
\mathrm{NDCG}_D(f,S_n) \overset{p}{\to}
\frac{(1-\beta)\int_0^1{\overline{y}^f(s)\cdot (1-s)^{-\beta}\ud
s}}{p^{1-\beta}}.} \label{power_converge}
\end{equation}
\end{theorem}

The proof will be given in Appendix
\ref{section:proof_feasible_discount}.\\

For $D(r) = r^{-\beta}$ ($\beta \in (0,1)$), NDCG no longer
converges to the same limit for all ranking functions. The limit is
actually a correlation between $\overline{y}^f(s)$ and
$(1-s)^{-\beta}$. For a good ranking function $f$,
$\overline{y}^f(s)=\Pr[Y=1|\tilde{f}(X)=s]$ is likely to be an
increasing function of $s$, and thus has positive correlation with
$(1-s)^{-\beta}$. Therefore, the limit of the ranking measure
already differentiate good and bad ranking functions to some extent.

We next study whether NDCG with polynomial discount has power of
distinguishability as strong as the standard NDCG. That is, we will
see if Theorem \ref{theorem:comparable} holds for NDCG with
$r^{-\beta}$ ($\beta \in (0,1)$).

\begin{theorem}\label{theorem:comparable_r^beta}
Let $D(r) = r^{-\beta}$, $\beta \in (0,1)$. Assume $ p= \Pr[Y=1] >
0$. For every pair of ranking functions $f_0$, $f_1$, denote
$\overline{y}^{f_i}(s) = \Pr[Y=1|\tilde{f}_i(X)=s]$, $i=0,1$, and
$\Delta y(s) = \overline{y}^{f_0}(s) - \overline{y}^{f_1}(s)$.
Suppose at least one of the following two conditions hold: 1)
$\int_0^1 \Delta y(s) (1-s)^{-\beta} \ud s \neq 0$; 2)
$\overline{y}^{f_0}(s)$, $\overline{y}^{f_1}(s)$ are H\"{o}lder
continuous with H\"{o}lder continuity constant $\alpha$ satisfying
$\alpha
> 3(1- \beta)$, and $\Delta y (1) \neq 0$. Then $f_0$ and
$f_1$ are strictly distinguishable with high probability by NDCG
with discount $D(r)$.
\end{theorem}

The proof will be given in Appendix
\ref{section:proof_comparable_r^beta}.\\

Theorem \ref{theorem:comparable_r^beta} involves two conditions.
Satisfying either of them leads to strictly distinguishable with
high probability. The first condition simply means that
$\mathrm{NDCG}_D(f_0,S_n)$ and $\mathrm{NDCG}_D(f_1,S_n)$ converge
to different limits and therefore the two functions are
consistently distinguishable in the strongest sense. The second condition deals
with the case that $\mathrm{NDCG}_D(f_0,S_n)$ and
$\mathrm{NDCG}_D(f_1,S_n)$ converge to the same limit. Comparing the
distinguishability of NDCG with $r^{(-\beta)}$ discount with the
standard NDCG, in most cases $r^{(-\beta)}$ discount has stronger
distinguishability than standard NDCG (i.e., when the measures of
two ranking functions converge to different limits). On the other hand, if we
consider the worst case, standard NDCG is better, because it
requires less conditions for consistent distinguishability.


We next study the Zipfian discount $D(r) = r^{-1}$. The following
theorem describes the limit of the ranking measure.

\begin{theorem}\label{theorem:r^-1}\label{thm_r1}
Assume $D(r) = r^{-1}$. Assume also $ p= \Pr[Y=1] > 0$ and
$\overline{y}^f(s) = \Pr[Y=1|\tilde{f}(X)=s]$ is a continuous
function. Then
\begin{equation}{
\mathrm{NDCG}_D(f,S_n) \overset{p}{\to} \Pr[Y=1|\tilde{f}(X)=1].}
\end{equation}
\end{theorem}

The proof of Theorem \ref{theorem:r^-1} will be given in Appendix
\ref{section:proof_feasible_discount}.\\

The limit of NDCG with Zipfian discount depends only on the
performance of the ranking function for the top ranks. The relevancy
of lower ranked items does not affect the limit.

The next logical step would be analyzing the power of
distinguishability of NDCG with Zipfian discount. However we are not
able to prove that consistent distinguishability holds for this ranking measure. The techniques developed for
distinguishability theorems given above does not apply to the
Zipfian discount. Although we cannot disprove it distinguishability,
we suspect that Zipfian does not have strong consistent distinguishability
power.

Finally, we consider discount functions that decay substantially
faster than $r^{-1}$. We will show that with these discount, NDCG
does not converge as the number of objects tends to infinity. More
importantly, such NDCG does not have the desired consistent distinguishability
property.

\begin{theorem}\label{prop_unbound}
Let $\mathcal X$ be instance space. For any $x\in\mathcal X$, let
$y_x^* = \mathrm{argmax}_{y\in\mathcal Y}{\Pr(Y=y|X=x)}$. Assume
that there is an absolute constant $\delta > 0$ such that for every
$x\in\mathcal X$, $\Pr(Y=y|X=x)\geq \delta\cdot \Pr(Y=y_x^*|X=x)$
for all $y\in\mathcal Y$. If $\sum_{r=1}^{\infty}{D(r)} \leq B$ for
some constant $B>0$, then $\mathrm{NDCG}_D(f,S_n)$ does not converge
in probability for any ranking function $f$. In particular, if $D(r)
\le r^{-(1+\epsilon)}$ for some $\epsilon >0$,
$\mathrm{NDCG}_D(f,S_n)$ does not converge. Moreover, every pair of
ranking functions are not consistently distinguishable by NDCG with such discount.
\end{theorem}

The proof is given in Appendix
\ref{section:proof_feasible_discount}.\\

Now we are able to \emph{characterize} the feasible discounts for
NDCG according to the results given so far. The logarithmic
$\frac{1}{\log(1+r)}$ and polynomial $r^{-\beta}$ ($\beta \in
(0,1)$) are feasible discount functions for NDCG. For different
ranking functions, standard NDCG converges to the same limit while
the $r^{-\beta}$ ($\beta \in (0,1)$) one converges to different
limits in most cases. However, if we ignore the numerical scaling
issue, both logarithmic and $r^{-\beta}$ ($\beta \in (0,1)$)
discount have consistent distinguishability. The Zipfian
$r^{-1}$ discount is on the borderline. It is not clear whether it
has strong power of distinguishability. Discount that decays faster
than $r^{-(1+\epsilon)}$ for some $\epsilon > 0$ is not appropriate
for NDCG when the data size is large.

\subsection{Cut-off Versions of NDCG}\label{subsection_NDCG@k}

In this section we study the top-$k$ version of NDCG, i.e.,
NDCG@$k$. For NDCG@$k$, the discount function is set as $D(r) = 0$
for all $r>k$. The motivation of using NDCG@$k$ is to pay more
attention to the top-ranked results. Logarithmic discount is also
dominant for NDCG@k. We will call this measure standard NDCG@k. As
already stated in Introduction, a natural question of standard
NDCG@k is why use a combination of a very low logarithmic decay and
a hard cut-off as the discount function. Why not simply use a smooth
discount with fast decay, which seems more natural. In fact, this
question has already been answered by Theorem \ref{prop_unbound}.
NDCG with such discount does not have strong power of
distinguishability.

We next address the issue that how to choose the cut-off threshold
$k$. It is obvious that setting $k$ as a constant independent of $n$
is not appropriate, because the partial sum of the discount is
bounded and according to Theorem \ref{prop_unbound} the ranking
measure does not converge. So $k$ must grow unboundedly as $n$ goes
to infinity. Below we investigate the convergence and
distinguishability of NDCG@k for various choices of $k$ and the
discount function. For clarity reason we assume here $\mc{Y} =
\{0,1\}$, and general results will be given in Section
\ref{subsection:General_Cases}. The proofs of all theorems in this
section will be given in Appendix
\ref{section:proof_feasible_discount}. We fist consider the case
$k=o(n)$.

\begin{theorem}
Let $\mathcal Y=\{0,1\}$. Assume $D(r)$ is a discount function and
$\sum_{r=1}^{\infty}{D(r)}$ is unbounded. Suppose $k = o(n)$ and $k
\rightarrow \infty$ as $n \rightarrow \infty$. Let $\tilde D(r) =
D(r)$ for all $r\leq k$ and $\tilde D(r) = 0$ for all $r > k$.
Assume also that $p= \Pr[Y=1]
> 0$ and $\overline{y}^f(s) = \Pr[Y=1|\tilde{f}(X)=s]$ is a continuous
function. Then
\begin{equation}{
\mathrm{NDCG}_{\tilde D}(f,S_n)\overset{p}{\to}
\Pr[Y=1|\tilde{f}(X)=1].}
\end{equation}
\label{theorem:kon}
\end{theorem}

The limit of NDCG@k where $k=o(n)$ is exactly the same as NDCG with
Zipfian discount. Also like the Zipfian, the distinguishability
power of this NDCG@k measure is not clear.

We next consider the case $k = cn$ for some constant $c\in(0,1)$. We
study the standard logarithmic and the polynomial discount
respectively in the following two theorems.

\begin{theorem}
Assume $D(r) = \frac{1}{\log(1+r)}$ and $\mathcal Y=\{0,1\}$. Let $k
= cn$ for some constant $c\in (0,1)$. Define the cut-off discount
function $\tilde D$ as $\tilde D(r) = D(r)$ if $r\leq k$ and $\tilde
D(r) = 0$ otherwise. Assume also $p= \Pr[Y=1] > 0$ and
$\overline{y}^f(s) = \Pr[Y=1|\tilde{f}(X)=s]$ is a continuous
function. Then
\begin{equation}{
\mathrm{NDCG}_{\tilde D}(f,S_n)\overset{p}{\to} \frac{c}{\min\{c,
p\}}\cdot\Pr[Y=1|\tilde{f}(X)\geq 1-c].} \label{thm_kcn_NDCG_eq}
\end{equation}
\label{theorem:kcn_NDCG}
\end{theorem}

\begin{theorem}
Assume $D(r) = r^{-\beta}$ and $\mathcal Y=\{0,1\}$, where
$\beta\in(0,1)$. Let $k=cn$ for some constant $c\in (0,1)$. Define
the cut-off discount function $\tilde D(r) = D(r)$ if $r\leq k$ and
$\tilde D(r) = 0$ otherwise. Assume also $p= \Pr[Y=1] > 0$ and
$\overline{y}^f(s) = \Pr[Y=1|r(X)=s]$ is a continuous function. Then
\begin{equation}{
\mathrm{NDCG}_{\tilde D}(f,S_n)\overset{p}{\to}
\frac{1-\beta}{(\min\{c,p\})^{1-\beta}}\cdot
\int_{1-c}^1{\overline{y}^f(s)\cdot (1-s)^{-\beta}\ud s}.}
\label{thm_kcn_NDCG}
\end{equation}
\label{theorem:kcn_poly}
\end{theorem}

The consistent distinguishability of the two measures considered in Theorem
\ref{theorem:kcn_NDCG} and Theorem \ref{theorem:kcn_poly} are
similar to their corresponding full NDCG respectively. To be
precise, for NDCG@k ($k=cn$) with logarithmic discount and NDCG@k
with $r^{-\beta}$ ($\beta \in (0,1)$) discount, consistent
distinguishability holds under the condition
given in Theorem \ref{theorem:comparable} and Theorem
\ref{theorem:comparable_r^beta} respectively. Hence these two
cut-off versions NDCG are feasible ranking measures.

\subsection{Results for General $\mathcal{Y}$}\label{subsection:General_Cases}
Some theorems given so far assume $\mathcal{Y}=\{0,1\}$. Here we
give complete results for the general case that $|\mathcal{Y}| \ge
2$, and $\mathcal{Y}=\{\mathfrak{y}_1,\ldots,\mathfrak{y}_{ |
\mathcal{Y} | }\}$. We only state the theorems and omit the proofs,
which are straightforward modifications of the special case $\mc{Y}=
\{0,1\}$. The case $D(r)=\frac{1}{\log (1+r)}$ has already been
included in Theorem \ref{theorem:converge_2_1}. It always converges
to $1$ whatever the ranking function is. We next consider
$r^{-\beta}$ decay.

\begin{theorem}
Assume $D(r)=r^{-\beta}$ with $\beta \in (0,1)$. Suppose that
$\mathcal{Y}=\{\mathfrak{y}_1,\ldots,\mathfrak{y}_{ | \mathcal{Y} |
}\}$, where $\mathfrak{y}_1 > \ldots >
\mathfrak{y}_{|\mathcal{Y}|}$. Assume $f(X) \in [a,b]$; $f(X)$ has a
probability density function such that $\mathbb P(f(X)=s) >0$ for
all $s \in [a,b]$; $\Pr(Y=\mathfrak{y}_j)>0$ and
$\Pr(Y=\mathfrak{y}_j|\tilde{f}(X)=s)$ is a continuous function of
$s$ for all $j$. Then
\begin{equation*}
\mathrm{NDCG}_D(f,S_n) \xrightarrow{p}
\frac{(1-\beta)\int_{0}^{1}\mathbb{E}[Y|\tilde{f}(X)=s](1-s)^{-\beta}ds}{\sum_{j=1}^{|\mathcal{Y}|}\mathfrak{y}_j(R_j^{1-\beta}-R_{j-1}^{1-\beta})}
\end{equation*}
where $R_0=0$; $R_j = \Pr(Y \ge \mathfrak{y}_j)$.
\end{theorem}

The next theorem is for top-k type NDCG measures, where $k=o(n)$.

\begin{theorem}
Suppose that $\mathcal{Y}=\{\mathfrak{y}_1,\ldots,\mathfrak{y}_{|
\mathcal{Y} |}\}$, where $\mathfrak{y}_1 > \ldots >
\mathfrak{y}_{|\mathcal{Y}|}$. Assume $ D(r)$ and $k$ grow
unboundedly and $k/n=o(1)$. For any $n$, let $\tilde D(r) = D(r)$ if
$r \le k$ and $\tilde D(r)=0$ otherwise. Assume $f(X) \in [a,b]$;
$f(X)$ has a probability density function such that $\mathbb
P(f(X)=s) >0$ for all $s \in [a,b]$; $\Pr(Y=\mathfrak{y}_j)>0$ and
$\Pr(Y=\mathfrak{y}_j|f(X)=s)$ is a continuous function of $s$ for
all $j$. Then
\begin{equation*}
\mathrm{NDCG}_{\tilde D}(f,S_n) \xrightarrow{p}
\frac{1}{\mathfrak{y}_1} \cdot \mathbb{E}[Y|\tilde{f}(X)=1].
\end{equation*}
\end{theorem}

The last two theorems are for top-$k$, where $k/n=c$. We consider
both logarithm discount and polynomial discount separately.

\begin{theorem}\label{theorem:complete_kon}
Suppose that $\mathcal{Y}=\{\mathfrak{y}_1,\ldots,\mathfrak{y}_{|
\mathcal{Y} |}\}$, where $\mathfrak{y}_1 > \ldots >
\mathfrak{y}_{|\mathcal{Y}|}$. Let $k/n=c$ for some constant $c>0$.
Let $D(r)=\frac{1}{\log(1+r)}$. For any $n$, let $\tilde D(r) =
D(r)$ if $r \le k$ and $\tilde D(r)=0$ otherwise. Assume $f(X) \in
[a,b]$; $f(X)$ has a probability density function such that $\mathbb
P(f(X)=s)>0$ for all $s \in [a,b]$; $\Pr(Y=\mathfrak{y}_j)>0$ and
$\Pr(Y=\mathfrak{y}_j|f(X)=s)$ is a continuous function of $s$ for
all $j$. Then
\begin{equation*}
\mathrm{NDCG}_{\tilde D}(f,S_n) \xrightarrow{p} \frac{c \cdot
\mathbb{E}[Y|\tilde{f}(X) \ge
1-c]}{\sum_{j=1}^{t}\mathfrak{y}_j(R_j-R_{j-1})+\mathfrak{y}_{t+1}(c-R_t)}.
\end{equation*}
where $R_0=0$; $R_j = \mathbb P(Y \ge \mathfrak{y}_j)$; t is defined
by $R_t < c \le R_{t+1}$.
\end{theorem}

\begin{theorem}
Let $ D(r)=r^{-\beta}$ with $\beta \in (0,1)$, and $\tilde D(r) =
D(r)$ if $r \le k$ and $\tilde D(r)=0$ otherwise. Using the same
notions and under the same conditions as in Theorem
\ref{theorem:complete_kon}
\begin{equation*}
\mathrm{NDCG}_{\tilde D}(f,S_n) \xrightarrow{p}
\frac{(1-\beta)\int_{1-c}^{1}\mathbb{E}[Y|\tilde{f}(X)=s](1-s)^{-\beta}ds}{\sum_{j=1}^{t}\mathfrak{y}_j(R_j^{1-\beta}-R_{j-1}^{1-\beta})+\mathfrak{y}_{t+1}(c^{1-\beta}-R_t^{1-\beta})}.
\end{equation*}
\end{theorem}

\section{Experimental Results}\label{section:experiments}

All theoretical results in this paper are proved under the
assumption that the objects to rank are i.i.d. data. Often in real
applications the data are not strictly i.i.d or even not random.
Here we conduct experiments on a real dataset --- Web search data.
The aim is to see to what extent the behavior of the ranking
measures on real datasets agree with our theory obtained under the
i.i.d. assumption.

The dataset we use contains click-through log data of a search
engine. We collected the clicked documents for 40 popular queries as
test set, which are regarded as $40$ independent ranking tasks. In
each task, there are $5000$ Web documents with clicks. To avoid
heavy work of human labeling, we simply label each document by its
click number according to the following rule. We assign relevancy
$y=2$ to documents with more than $1000$ clicks, $1$ to those with
$100$ to $1000$ clicks, and $0$ to the rest. In each task, we
extracted $40$ features for each item representing its relevance to
the given query. A detail is how to construct $S_n$. In our
theoretical analysis we assume $S_n$ contains i.i.d. data. Since the
goal of the experiments is to see how our theory works for real
applications, we construct $S_n$ as follows. For each query, there
are totally $5000$ documents which we denote by
$x_1,\ldots,x_{5000}$. Assume each document has a generating time.
Without loss of generality we assume $x_1$ was generated earliest
and $x_{5000}$ latest. We set $S_n=\{(x_1,y_1),\dots,(x_n,y_n)\}$
for each $1\le n \le 5000$. Such a construction simulates that in
reality there may be increasing number of documents needed to rank
by a search engine over time. We use three ranking functions in the
experiments: a trained RankSVM model \citep{RSVM}, a trained ListNet
model \citep{listnet}, and a function chosen randomly. To be
concrete, the random function is constructed as follows. For each $x
\in \mc{X}$, we set $f(x)$ by choosing a number uniformly random
from $[-1,1]$. For the trained models (i.e., listNet and RankSVM),
parameters are learned from a separate large training set construct
in the same manner as the test set. Clearly, ListNet and RankSVM are
relatively good ranking functions and the random function is bad.

We analyze the following typical NDCG type ranking measures by
experiments:
\begin{itemize}
\item Standard NDCG: $D(r) =
\frac{1}{\log(1+r)}$. See Figure \ref{Figure_1/log(1+r)}, Theorem
\ref{theorem:converge_2_1} and Theorem \ref{theorem:comparable}.

\item NDCG with a feasible discount function: $D(r) =
r^{-1/2}$. See Figure \ref{Figure_r^(-1/2)}, Theorem
\ref{theorem:r^-alpha} and Theorem \ref{theorem:comparable_r^beta}.

\item NDCG with too fast decay: $D(r) =
2^{-r}$. See Figure \ref{Figure_2^(-r)} and Theorem
\ref{prop_unbound}.

\item NDCG@k: $k=n/5$; $D(r) =
\frac{1}{\log(1+r)}$. See Figure \ref{Figure_NDCG@k=n/5} and Theorem
\ref{theorem:kcn_NDCG}.
\end{itemize}

\begin{figure}
  \centering
      \includegraphics[width=0.8\textwidth]{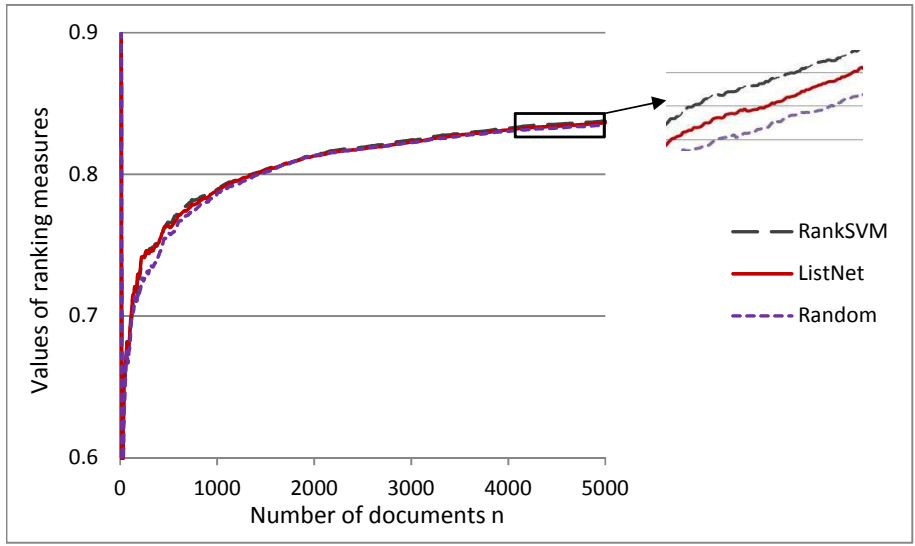} \caption{Standard NDCG: Converges to the same limit but distinguishes well the ranking functions.}\label{Figure_1/log(1+r)}
\end{figure}

\begin{figure}
  \centering
      \includegraphics[width=0.8\textwidth]{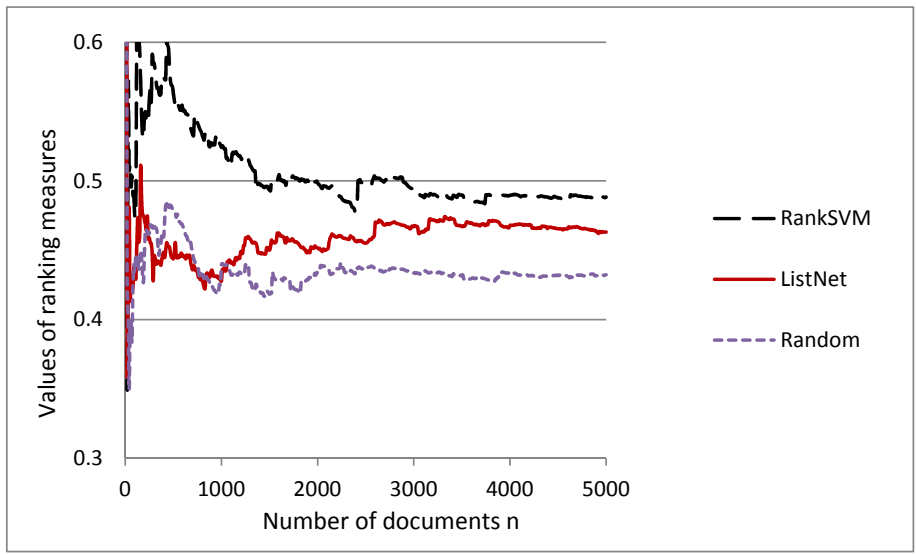} \caption{NDCG with feasible discount $D(r) = r^{-1/2}$: converges to different limits and distinguishes well the ranking functions.}\label{Figure_r^(-1/2)}
\end{figure}

\begin{figure}
  \centering
      \includegraphics[width=0.8\textwidth]{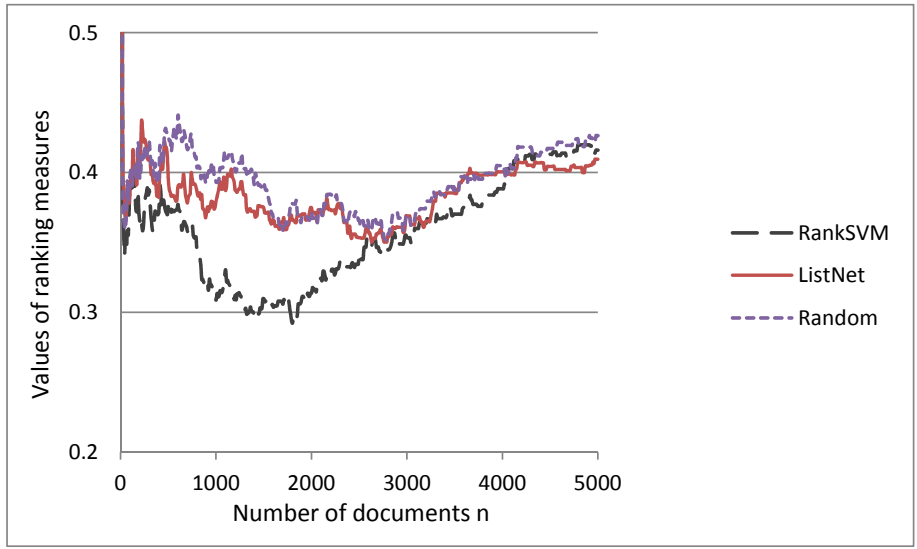} \caption{NDCG with too fast decay $D(r) = 2^{-r}$: does not converge; does not have good distinguishability power either.}\label{Figure_2^(-r)}
\end{figure}

\begin{figure}
  \centering
      \includegraphics[width=0.8\textwidth]{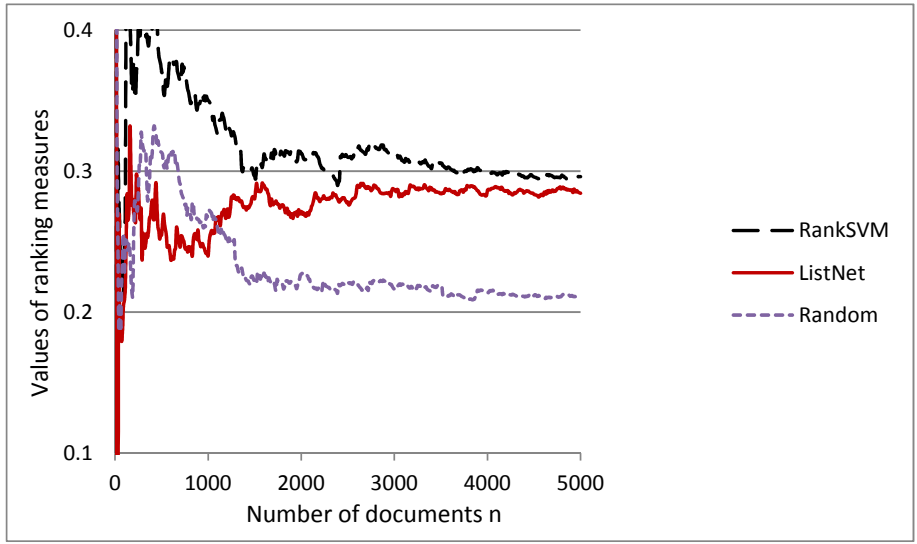} \caption{NDCG@$k$ ($D(r) = \frac{1}{\log(1+r)}$, $k=n/5$): distinguishes well the ranking functions.}\label{Figure_NDCG@k=n/5}
\end{figure}

Figure \ref{Figure_1/log(1+r)} agrees well with Theorem
\ref{theorem:converge_2_1} and Theorem \ref{theorem:comparable}. On
the one hand, the NDCG measures of the three ranking functions are
very close and seem to converge to the same limit. On the other
hand, one can see from the enlarged part (we enlarge and stretch the
vertical axis) in the figure that in fact the measures distinguish
well the ranking functions.

Figure \ref{Figure_r^(-1/2)} demonstrates the result of NDCG with
the feasible discount $r^{-1/2}$. In this experiment, it seems that
the ranking measures of the three ranking functions converge to
different limits and therefore distinguish them very well. In our
experimental setting, it is not easy to find two ranking functions
whose NDCG measures converge to the same limit. If one can find such
a pair of ranking functions, it would be interesting to see how well
the measure distinguish them.

Figure \ref{Figure_2^(-r)} shows the behavior of NDCG with a smooth
discount which decays too fast. The measure cannot distinguish the
three ranking functions very well. Even the randomly chosen function
has an NDCG score similar to those of RankSVM and ListNet. From the
figure, it is also likely that the measures do not converge.

Figure \ref{Figure_NDCG@k=n/5} depicts the result of NDCG@k, where
$k$ is a constant proportion of $n$. Before describing the result,
let us first comparing Theorem \ref{theorem:kcn_NDCG} and Theorem
\ref{theorem:converge_2_1}. Note that although the discount are both
the logarithmic one, NDCG@k for $k=cn$ can converge to different
limits for different ranking functions, while standard NDCG always
converges to $1$. Figure \ref{Figure_NDCG@k=n/5} clearly demonstrate
this result.

\section*{Acknowledgement}
Liwei Wang would like to thank Kai Fan and Ziteng Wang for long and helpful discussions.

\bibliographystyle{abbrv}

\bibliography{NDCG}

\appendix

\section{Proof of Theorem \ref{theorem:comparable}: the Key Lemmas}\label{Section:key_lemmas}

In this section we will prove Theorem \ref{theorem:comparable}. In fact we will prove a more complete result. The proof relies on a few key lemmas. In this section we only state these lemmas. Their proofs will be given in Appendix \ref{Section:proof_key_lemmas}. First we give a weaker definition of distinguishability, which guarantees that the ranking measure $\mc{M}$
gives consistent comparison results for two ranking functions only
in expectation.

\begin{definition}\label{definition:comparable_in_expectation}
Fix an underlying distribution $P_{XY}$. A pair of ranking functions
$f_0$, $f_1$ is said to be \emph{distinguishable in
expectation} by a ranking measure $\mc{M}$, if there exist $b \in
\{0,1\}$ and a positive integer $N$ such that for all $n \ge N$,
\begin{equation*}\small{
\mathbb{E}\big[\mc{M}(f_b,S_n)\big] >
\mathbb{E}\big[\mc{M}(f_{1-b},S_n)\big],}
\end{equation*}
where the expectation is over the random draw of $S_n$.
\end{definition}

Now we state a theorem which contains Theorem \ref{theorem:comparable}.

\begin{theorem}\label{theorem:comparable_complete}
Assume that $p=\Pr(Y=1)>0$. For every pair of ranking functions
$f_0, f_1$, Let $\overline y^{f_i}(s) = \Pr[Y=1|\tilde{f}_i(X)=s]$,
$i=0,1$. Unless $\overline y^{f_0}(s) = \overline y^{f_1}(s)$ almost
surely on $[0,1]$, $f_0, f_1$ are distinguishable in
expectation by standard NDCG whose discount is $D(r) = \frac{1}{\log
(1+r)}$.

Moreover, if $\overline y^{f_0}(s)$ and $\overline y^{f_1}(s)$
are H\"{o}lder continuous in $s$, then unless $\overline y^{f_0}(s)
= \overline y^{f_1}(s)$ almost everywhere on $[0,1]$, $f_0$ and
$f_1$ are consistently distinguishable by standard
NDCG.
\end{theorem}

To prove Theorem \ref{theorem:comparable_complete}, we need some notations.

\begin{definition}\label{definition:pseudo_expectation}
Suppose $\mc{Y} = \{0,1\}$. Let $\overline y^{f}(s) =
\Pr[Y=1|\tilde{f}(X)=s]$. Also let $F(t) = \int_1^t D(s) ds$. We
define the \emph{unnormalized pseudo-expectation} $\tilde{N}_D^f(n)$
as
\begin{equation*}{
\tilde{N}_D^f(n) = \int_1^n{\overline{y}^f(1-s/n)D(s)\ud s} =
n\int_{\frac{1}{n}}^1{\overline{y}^f(1-s)D(ns)\ud s}.}
\end{equation*}
Assume that $p=\Pr(Y=1)>0$. Define the \emph{normalized
pseudo-expectation} $N^f_D(n)$ as
\begin{equation*}{
N^f_D(n)=\frac{\tilde{N}_D^f(n)}{F(np)}.}
\end{equation*}
\end{definition}

The proof of the first part of Theorem \ref{theorem:comparable_complete}
(i.e., distinguishable in expectation) relies on the following two
key lemmas, whose proofs will be given in Appendix
\ref{Section:proof_key_lemmas}.

\begin{lemma}\label{lemma:lemma1_for_comparable}
Let $D(r) = \frac{1}{\log (1+r)}$. Assume that $p=\Pr(Y=1)>0$. Then
for every ranking function $f$,
\begin{equation*}{
\left|\mathbb{E}[\mathrm{NDCG}_D(f,S_n)] - N_D^f(n) \right| \le
\tilde{O}\left( n^{-1/3} \right).}
\end{equation*}
\end{lemma}

\begin{lemma}\label{lemma:lemma2_for_comparable}
Let $D(r) = \frac{1}{\log (1+r)}$. Assume that $p=\Pr(Y=1)>0$. let
$\overline{y}^{f_i}(s) = \Pr[Y=1|\tilde{f}_i(X)=s]$, $i=0,1$. Unless
$\overline{y}^{f_0}(\cdot)=\overline{y}^{f_1}(\cdot)$ almost
everywhere on $[0,1]$, there must exist a nonnegative integer $K$
and a constant $a \neq 0$, such that
\begin{equation*}{
\Bigg|N_D^{f_0}(n) - N_D^{f_1}(n) - \frac{a}{\log^K n}\Bigg| \leq
O\left(\frac{1}{\log^{K+1} n}\right).}
\end{equation*}
\end{lemma}

Lemma \ref{lemma:lemma1_for_comparable} says that the difference
between the expectation of the NDCG measure of a ranking function
and its pseudo-expectation is relatively small; while Lemma
\ref{lemma:lemma2_for_comparable} says that the difference between
the pseudo-expectations of two essentially different ranking
functions are much larger.

To prove the ``moreover'' part of Theorem
\ref{theorem:comparable_complete} (i.e., consistently distinguishable), we need the following key lemma, whose proof will be
given in Section \ref{Section:proof_key_lemmas}. The lemma states
that with high probability the NDCG measure of a ranking function is
very close to its pseudo-expectation.

\begin{lemma}\label{lemma:lemma3_of_comparable}
Let $D(r)=\frac{1}{\log (1+r)}$. Assume that $p=\Pr(Y=1)>0$. Suppose
the ranking function $f$ satisfies that
$\overline{y}^f(s)=\Pr(Y=1|\tilde{f}(X)=s)$ is H\"{o}lder continuous
with constants $\alpha>0$ and $C>0$. That is,
$|\overline{y}^f(s)-\overline{y}^f(s')| \le C|s-s'|^{\alpha}$ for
all $s,s' \in [0,1]$. Then
\begin{equation*}{
\Pr\bigg[\Big|\mathrm{NDCG}_D(f,S_n)-N^f_D(n)\Big| \ge
5Cp^{-1}n^{-\min(\alpha/3,1)} \bigg] \le
O\left(e^{-n^{1/4}}\right).}
\end{equation*}
\end{lemma}

\begin{proof}\textbf{of Theorem \ref{theorem:comparable_complete}}
That $f_0$ and $f_1$ are strictly distinguishable in expectation by
standard NDCG is straightforward from Lemma
\ref{lemma:lemma1_for_comparable} and Lemma
\ref{lemma:lemma2_for_comparable}. That $f_0$ and $f_1$ are strictly
distinguishable with high probability follows immediately from Lemma
\ref{lemma:lemma3_of_comparable}, Lemma
\ref{lemma:lemma2_for_comparable} and the observation that $\sum_{n
\ge N} e^{-n^{1/4}} \le O\left(N^{3/4}e^{-N^{1/4}}\right) \le
O\left(e^{-N^{1/5}}\right).$
\end{proof}

\section{Proofs of the Key Lemmas in Appendix \ref{Section:key_lemmas}}\label{Section:proof_key_lemmas}

In this section, we give proofs of the three key lemmas in Appendix \ref{Section:key_lemmas} (i.e., Lemma
\ref{lemma:lemma1_for_comparable}, Lemma
\ref{lemma:lemma2_for_comparable} and Lemma
\ref{lemma:lemma3_of_comparable}) used to prove Theorem
\ref{theorem:comparable} and Theorem \ref{theorem:comparable_complete}.

To prove the key lemmas, we need a few technical claims, whose
proofs will be given in Appendix \ref{section:tech_lemmas}. We first
give four claims that will be used in the proof of Lemma
\ref{lemma:lemma1_for_comparable}.

\begin{claim}\label{claim_uniform}
For any $s \in [0,1]$,
\begin{equation}\label{uniform}{
\sum_{r=1}^n \mathbb{P} \big[ \tilde{f}(x^f_{(r)}) = s \big] = n.}
\end{equation}
\end{claim}

\begin{claim}\label{claim_dcg_expectation}
Recall that the DCG ranking measure with respect to discount $D(\cdot)$
was defined as
\begin{equation}{
\mathrm{DCG}_D(f,S_n)=\sum_{r=1}^{n} y^f_{(r)}D(r).}
\end{equation}
Let $D(r) = \frac{1}{\log(1+r)}$, and $\overline{y}^f(s) = \Pr[Y=1 |
\tilde{f}(X)=s]$. Then
\begin{equation}\label{dcg_expectation_eq1}{
\mathbb{E}\left[\mathrm{DCG}_D(f,S_n)\right] = \sum_{r=1}^{n}
\frac{1}{\log(1+r)}\int_0^1 \mathbb{P}\left[\tilde{f}(x^f_{(r)}) =
1-s \right]\overline{y}^f(1-s) \ud s.}
\end{equation}
\end{claim}

\begin{claim}\label{claim_concentration}
For any positive integer $n$, define $E_{n,r} =
[\frac{r}{n}-n^{-1/3},\frac{r}{n}+n^{-1/3}]$ ($r \in [n]$). Then for
any $r \in [n]$,
\begin{equation}{
\Pr\big[1- \tilde{f}(x^f_{(r)}) \in E_{n,r} \big] \ge
1-2e^{-n^{1/3}}.}
\end{equation}
\end{claim}

\begin{claim}\label{prop_dcg_NDCG}
Let $\mathcal Y=\{0,1\}$. Assume $D(r) = \frac{1}{\log (1+r)}$. Let
$F(t) = \int_1^t{D(s)\ud s}$. Assume also $p= \Pr[Y=1] > 0$. Then
for every sufficiently large $n$, with probability
$(1-2e^{-2n^{1/3}})$ the following inequality holds.
\begin{equation}{
\Bigg| \mathrm{NDCG}_D(f,S_n) -
\frac{\mathrm{DCG}_D(f,S_n)}{F(np)}\Bigg| \leq
O\left(n^{-1/3}\right).} \label{dcg_NDCG_approx}
\end{equation}
\end{claim}

Now we are ready to prove Lemma \ref{lemma:lemma1_for_comparable}.
\\

\begin{proof} \textbf{of Lemma \ref{lemma:lemma1_for_comparable}.}
By the definition of $\tilde{N}^f_D(n)$ (see Definition
\ref{definition:pseudo_expectation}) and eq.(\ref{uniform}), we have
\begin{equation}\label{proof_keylemma1_1}{
\tilde N_D^{f}(n) =
n\int_{\frac{1}{n}}^1{\frac{\overline{y}^f(1-s)\ud s}{\log(1+ns)}} =
\sum_{r=1}^n{\int_{\frac{1}{n}}^1{\frac{\overline{y}^f(1-s)\mathbb{P}[\tilde
f(x^f_{(r)})=1-s]}{\log(1+ns)}\ud s}}.}
\end{equation}
By eq. (\ref{dcg_expectation_eq1}) in Claim
\ref{claim_dcg_expectation} and eq.(\ref{proof_keylemma1_1}), and
note that $\overline{y}^f(s) \le 1$, we obtain
\begin{eqnarray}\label{approx_eq3}
&& \Big|\mathbb E[\mathrm{DCG}_D(f,S_n)] - \tilde N_D^f(n)\Big|
\nonumber \\
& \leq &
\sum_{r=1}^n{\Bigg|\int_{\frac{1}{n}}^1{\overline{y}^f(1-s)\mathbb{P}\big[\tilde
f(x^f_{(r)})=1-s\big]\left(\frac{1}{\log(1+r)} -
\frac{1}{\log(1+ns)}\right)\ud s}\Bigg|} + \frac{1}{n}\sum_{r=1}^n \frac{1}{\log(1+r)} \nonumber \\
& \leq & \sum_{r=1}^n{\Bigg|\int_{[\frac{1}{n},1]\backslash
E_{n,r}}{\mathbb{P}\big[\tilde
f(x^f_{(r)})=1-s\big]\left|\frac{1}{\log(1+r)} -
\frac{1}{\log(1+ns)}\right|\ud s}}  \nonumber \\
&+& \sum_{r=1}^n{\int_{E_{n,r}\cap [\frac{1}{n},1]}{\Bigg|\frac{1}{\log(1+r)} -
\frac{1}{\log(1+ns)}\Bigg|\ud s}} + O\left(\frac{1}{\log n}\right).
\end{eqnarray}

We next bound the two terms in the RHS of the last inequality of
(\ref{approx_eq3}) separately. By Claim \ref{claim_concentration},
the first term can be upper bounded by
\begin{equation}{
2e^{-n^{1/3}} \sum_{r=1}^n \sup_{s \in [\frac{1}{n},1]\backslash
E_{n,r}}\Bigg|\frac{1}{\log(1+r)} - \frac{1}{\log(1+ns)}\Bigg|  \leq
\frac{2}{\log 2}ne^{-2n^{1/3}}.} \label{approx_eq5}
\end{equation}
For the second term in the RHS of the last inequality of
(\ref{approx_eq3}), it is easy to check that the following two
inequalities hold:
\begin{equation}{
\forall r>n^{2/3}, ~\sup_{s\in E_{n,r}\cap
[\frac{1}{n},1]}{\Bigg|\frac{1}{\log(1+r)} -
\frac{1}{\log(1+ns)}\Bigg|} \leq \frac{n^{2/3}}{(1+r)\log^2(1+r)} +
o\left(\frac{n^{2/3}}{(1+r)\log^2(1+r)}\right).} \label{approx_eq6}
\end{equation}
\begin{equation}{
\forall r\leq n^{2/3}, \sup_{s\in E_{n,r}\cap
[\frac{1}{n},1]}{\Bigg|\frac{1}{\log(1+r)} -
\frac{1}{\log(1+ns)}\Bigg|} \leq \frac{1}{\log 2}.}
\label{approx_eq7}
\end{equation}
Combining (\ref{approx_eq3}), (\ref{approx_eq5}), (\ref{approx_eq6})
and (\ref{approx_eq7}), we obtain
\begin{equation}{
\Big|\mathbb E[\mathrm{DCG}_D(f,S_n)] - \tilde N_D^f(n)\Big| \leq
\frac{2ne^{-2ne^{1/3}}}{\log 2} + O\left(\frac{n^{2/3}}{\log 2} +
\sum_{r=n^{2/3}}^n{\frac{n^{2/3}}{(1+r)\log^2(1+r)}}\right) \leq \tilde
O\left(n^{2/3}\right).}
\end{equation}

Finally, observe that $F(np) = \li(1+np)$, where $\li$ is the offset
logarithmic integral function. By Claim \ref{prop_dcg_NDCG} and the
well-known fact $\li (n)\sim \frac{n}{\log n}$, we have the
following inequality and this completes the proof.
\begin{equation}{
\Bigg|\mathbb E[\mathrm{NDCG}_D(f,S_n)] - \frac{\mathbb E[
\mathrm{DCG}_D(f,S_n)]}{\li(1+np)}\Bigg| \leq \tilde
O\left(n^{-1/3}\right) + O\left(e^{-2n^{1/3}}\right).}
\label{approx_eq8}
\end{equation}
\end{proof}

We next turn to prove Lemma \ref{lemma:lemma2_for_comparable}. We
need the following three claims.

\begin{claim}
For sufficiently large $n$,
\begin{equation}{
\int_0^{\frac{2}{n}}{\log^k x\ud x} = O\left(\frac{\log^k
n}{n}\right).} \label{eq_int1}
\end{equation}
\label{prop_int1}
\end{claim}

\begin{claim}
Fix an integer $k\in\mathbb N^* = \{0\} \cup \mathbb N$. For
sufficiently large $n$,
\begin{equation}{
\int_{\frac{2}{n}}^1{\frac{\Big|\log^k x\Big|\ud
x}{(\log(nx))^{k+1}}} \leq O\left(\frac{1}{\log^{k+1} n}\right).}
\label{eq_int2}
\end{equation}
\label{prop_int2}
\end{claim}

\begin{claim}
$\mathrm{span}\left(\{\log^k x\}_{k \ge 0}\right)$, is dense in
$L^2[0,1]$. \label{prop_sep}
\end{claim}

Now we are ready to prove Lemma \ref{lemma:lemma2_for_comparable}.
\\

\begin{proof} \textbf{of Lemma \ref{lemma:lemma2_for_comparable}.}
Let $\Delta y(s) = \overline{y}^{f_0}(s) - \overline{y}^{f_1}(s)$.
By the definition of normalized pseudo expectation (see definition
\ref{definition:pseudo_expectation}) and the fact that $|\Delta y(s)|\leq 1$, we
have
\begin{eqnarray}
\nonumber N_D^{f_0}(n) - N_D^{f_1}(n) &=& \frac{n}{\li (1+np)}\int_{\frac{1}{n}}^1{\frac{\Delta y(1-s)\ud s}{\log (1+ns)}}\\
&=& \frac{n}{\li (1+np)}\int_{\frac{2}{n}}^1{\frac{\Delta y(1-s)\ud
s}{\log(1+ns)}} + O\left(\frac{1}{\li (n)}\right).
\label{separate_eq1}
\end{eqnarray}
Expanding $\frac{1}{\log(1+ns)}$ at the point $ns$, we obtain
\begin{equation}{
\Bigg|\int_{\frac{2}{n}}^1{\frac{\Delta y(1-s)\ud s}{\log(1+ns)}} -
\int_{\frac{2}{n}}^1{\frac{\Delta y(1-s)\ud s}{\log n + \log
s}}\Bigg| \leq \int_{\frac{2}{n}}^1{\frac{\ud s}{ns\log^2(ns)}} \leq
O\left(\frac{\log n}{n}\right).} \label{separate_eq2}
\end{equation}
Expanding $\frac{1}{\log n + \log s}$ at point $\log n$, we have
that for all $m\in \mathbb N^*$, the following holds:
\begin{eqnarray}
\nonumber && \Bigg|\int_{\frac{2}{n}}^1{\frac{\Delta y(1-s)\ud
s}{\log n+\log s}} - \sum_{j=1}^m{\frac{(-1)^{j-1}}{\log^j
n}\int_{\frac{2}{n}}^1{\Delta y(1-s)\log^{j-1} s~\ud s}}\Bigg| \\
&=& \Bigg|\int_{\frac{2}{n}}^1{\frac{\Delta y(1-s)\log^m s\ud
s}{(\log n + \xi_{n, s})^{m+1}}}\Bigg| \leq
\int_{\frac{2}{n}}^1{\frac{\Big|\Delta y(1-s)\log^ms\Big|\ud
s}{(\log n+\log s)^{m+1}}} \le O\left(\frac{1}{\log^{m+1} n}\right).
\label{separate_eq3}
\end{eqnarray}
Note in above derivation that $\xi_{n,s}\in (\log s, 0)$), and the
last inequality is due to Claim \ref{prop_int2}.

Furthermore, by Claim \ref{prop_sep}, unless $\Delta y(s) = 0$ a.e.,
there exist constants $k\in\mathbb N^*$ and $a \neq 0$ such that
\begin{equation}{
(-1)^k\int_0^1{\Delta y(1-s)\log^k s~\ud s} = a.}
\label{separate_eq4}
\end{equation}
Let $K$ be the smallest integer $k$ that Eq. (\ref{separate_eq4})
holds. Combining (\ref{separate_eq1}), (\ref{separate_eq2}),
(\ref{separate_eq3}), and (\ref{separate_eq4}) and noting Claim \ref{prop_int1}, we have the
following and this completes the proof.

\begin{equation*}{
\Bigg| N_D^{f_0}(n) - N_D^{f_1}(n) - \frac{a}{\log^K n}\Bigg| \leq
O\left(\frac{\log^K n}{n}\right) +
O\left(\frac{1}{\log^{K+1}n}\right).} \label{separate_eq5}
\end{equation*}

\end{proof}

To prove the last key lemma, we need the following claim.

\begin{claim}\label{Claim:for_comparable_lemma}
Let $D(r)=\frac{1}{\log (1+r)}$. Let $F(t) = \int_1^t D(r) \ud r$.
Assume $\overline{y}^f(s)$ is H\"{o}lder continuous with constants
$\alpha$ and $C$. Then
\begin{equation}{
\Big|\sum_{r=1}^n{\overline{y}^f(1-r/n)D_r} - \tilde N^f_D(n) \Big|
\le C n^{-\alpha / 3} F(n) + D(1) + |D'(1)|.}
\end{equation}
\end{claim}

Now we prove the last key lemma.\\

\begin{proof} \textbf{of Lemma \ref{lemma:lemma3_of_comparable}.}
Let $x_1,\cdots,x_n$ be instances i.i.d. drawn according to $P_{X}$.
Let $\tilde x_{(r)} = \tilde f(x_{(r)}^f)$ and by definition $\tilde
x_{(1)}\geq \tilde x_{(2)}\geq\cdots \geq \tilde x_{(n)}.$ By
Chernoff bound, for every $r$ with probability $2e^{-2n^{1/3}}$ we
have $|\tilde x_{(r)} - (1-r/n)| > n^{-1/3}$. A union bound over $r$
then yields
\begin{equation}{
\Pr\left[\forall r\in[n], \Big|\tilde x_{(r)} -
\left(1-\frac{r}{n}\right)\Big| \leq n^{-1/3}\right] \geq
1-2ne^{-2n^{1/3}}.} \label{mnclose_eq1}
\end{equation}

Since $y^f$ is H\"{o}lder continuous with constants $\alpha$ and
$C$, eq. (\ref{mnclose_eq1}) implies
\begin{equation}{
\Pr\left[\Big|\sum_{r=1}^n{\overline{y}^f(\tilde x_{(r)})D(r)} -
\sum_{r=1}^n{\overline{y}^f(1-r/n)D(r)}\Big| \leq
Cn^{-\alpha/3}\cdot \sum_{r=1}^n{D(r)}\right] \geq
1-2ne^{-2n^{1/3}}.} \label{mnclose_eq1.5}
\end{equation}

Combining Claim \ref{Claim:for_comparable_lemma} and eq.
(\ref{mnclose_eq1.5}), and note that $|D'(1)|+D(1) \le 10$ we have
\begin{equation}{
\Pr\left[\Big|\sum_{r=1}^n{\overline{y}^f(\tilde x_{(r)})D(r)} -
\tilde N_D^f(n)\Big| \leq 2Cn^{-\alpha/3}\cdot F(n)+10 \right] \geq
1-2ne^{-2n^{1/3}}.} \label{mnclose_eq2}
\end{equation}

Fix $x_1, \ldots, x_n$. Let $x^f_{(1)},\ldots, x^f_{(n)}$ be the
induced ordered sequence. Also let $\tilde x_{(r)} = \tilde
f(x^f_{(r)})$. Recall that $\overline{y}^f(s)= \mathbb{E}[Y|\tilde f
(X) =s]$. Thus $\sum_{r=1}^n \overline{y}^f(\tilde x_{(r)}) D(r)$ is
the expectation of $\mathrm{DCG}_D(f,S_n)=\sum_{r=1}^n y^f_{(r)}
D(r)$ conditioned on the fixed values $\tilde x_{(1)},\ldots, \tilde
x_{(n)}$. Also observe that conditioning on $\tilde x_{(1)},\ldots,
\tilde x_{(n)}$, $y^f_{(r)}$ ($r=1,\ldots,n$) are independent. By
Hoeffding's inequality and taking into consideration that
$x_1,\ldots,x_n$ are arbitrary and $(D(r))^2 \le D(r)$ for all $r$,
we have for every $\epsilon
> 0$
\begin{equation}{
\Pr\left[\Bigg|\mathrm{DCG}_D(f,S_n) -
\sum_{r=1}^n{\overline{y}^f(\tilde x_{(r)})D(r)}\Bigg| \geq
\epsilon\right] \leq 2\exp\left(-\frac{2\epsilon^2}{F(n)}\right).}
\label{mnclose_eq3}
\end{equation}
Set $\epsilon = F(n)^{2/3}$ in eq. (\ref{mnclose_eq3}) and combine
eq. (\ref{mnclose_eq2}), we have
\begin{equation}{
\Pr\left[\Big|\mathrm{DCG}_D(f,S_n) - \tilde N_D^f(n)\Big| >
2Cn^{-\alpha/3}F(n)+2F(n)^{2/3}\right] \leq 2ne^{-2n^{1/3}} +
2e^{-2F(n)^{1/3}}.} \label{mnclose_eq4}
\end{equation}
Simple calculations yields
\begin{equation}{
\Pr\left[\Bigg|\frac{\mathrm{DCG}_D(f,S_n)}{F(np)} - N_D^f(n)\Bigg|
> 4Cp^{-1}n^{-\min(\alpha/3,1)}\right] \leq 2ne^{-2n^{1/3}} +
2e^{-2F(n)^{1/3}}.} \label{mnclose_eq5}
\end{equation}
Combining eq. (\ref{dcg_NDCG_approx}) and (\ref{mnclose_eq5}) The
lemma follows.
\end{proof}


\section{Proof of the Technical Claims in Appendix \ref{Section:proof_key_lemmas}} \label{section:tech_lemmas}
Here we give proofs of the technical claims by which we prove the
three key lemmas in Section \ref{Section:proof_key_lemmas}.\\

\begin{proof} \textbf{of Claim \ref{claim_uniform}.}

Recall that for each $i \in [n]$, $\tilde{f}(x_i)$ is uniformly
distributed on $[0,1]$; and $x^f_{(1)},\ldots,x^f_{(n)}$ are just
reordering of $x_1,\ldots,x_n$. Thus
\begin{equation*}
\sum_{r=1}^n \mathbb{P}\big[ \tilde{f}(x^f_{(r)}) = s \big] =
\sum_{i=1}^n \mathbb{P}\big[ \tilde{f}(x_i) = s \big] = n.
\end{equation*}
\end{proof}

\begin{proof} \textbf{of Claim \ref{claim_dcg_expectation}.}

We have
\begin{eqnarray}
\mathbb{E}\left[\mathrm{DCG}_D(f,S_n)\right] &=&  \sum_{r=1}^n
D(r) \mathbb{E}\left[y^f_{(r)}\right] \nonumber \\
& = & \sum_{r=1}^n \frac{1}{\log(1+r)}
\mathbb{E}\bigg[\mathbb{E}\big[y^f_{(r)}| \tilde{f}(x^f_{(r)}) \big]
\bigg] \nonumber \\
& = & \sum_{r=1}^n \frac{1}{\log(1+r)}\int_0^1
\mathbb P\left[\tilde{f}(x^f_{(r)}) = s \right]\overline{y}^f(s) \ud s.
\end{eqnarray}
\end{proof}

\begin{proof} \textbf{of Claim \ref{claim_concentration}.}
Just observe that $\tilde{f}(x^f_{(r)})$ is the $r$-th order
statistic ($r$-th largest) of $n$ uniformly distributed random
variables on $[0,1]$. Chernoff bound yields the result.
\end{proof}

\begin{proof} \textbf{of Claim \ref{prop_dcg_NDCG}.}

Let $l = \sum_{(x,y)\in S_n}{\mathbb{I}[y=1]}$ be the number of
$y=1$ in $S_n$. Since $S_n$ is sampled i.i.d. and $\Pr[Y = 1] = p$,
by Chernoff bound we have
\begin{equation}
\Pr\left[\Big|l/n - p\Big| > n^{-1/3}\right] \leq 2e^{-2n^{1/3}}.
\label{prop_eq_1}
\end{equation}
Thus with probability at least $1-2e^{-2n^{1/3}}$
\begin{eqnarray*}
&& \Bigg|\mathrm{NDCG}_D(f,S_n) -
\frac{\mathrm{DCG}_D(f,S_n)}{F(np)}\Bigg| \nonumber \\
& = & \Bigg| \frac{\mathrm{DCG}_D(f,S_n)}{l}-
\frac{\mathrm{DCG}_D(f,S_n)}{F(np)}\Bigg| \nonumber \\ & \le &
\mathrm{DCG}_D(f,S_n)\cdot
\max\Bigg(\left|\frac{1}{F(n(p-n^{-1/3}))} - \frac{1}{F(np)}
\right|, \left|\frac{1}{F(n(p+n^{-1/3}))} - \frac{1}{F(np)} \right|
\Bigg). \label{prop_eq_3}
\end{eqnarray*}

Recall that $F(t) = \int_1^t \frac{1}{\log(1+r)} \ud r$, $p>0$; and
observe that $\mathrm{DCG}_D(f,S_n) \le F(n)$. Taylor expansion of
$\frac{1}{F((p\pm n^{-1/3})n)}$ at $np$ and some simple calculations
yields the result.

\end{proof}

\begin{proof} \textbf{of Claim \ref{prop_int1}.}

Integration by part we have,
\begin{equation}
\int{\log^kx\ud x} = k!\sum_{j=0}^k{(-1)^{k-j}\frac{x\log^jx}{j!}} +
C. \label{eq_int1}
\end{equation}
The claim follows.
\end{proof}

\begin{proof} \textbf{of Claim \ref{prop_int2}.}

Changing variable by letting $x=n^{-t}$ we have

\begin{eqnarray}\label{int2_eq1}
& & \int_{\frac{2}{n}}^{1}{\frac{\Big|\log^k x\Big|\ud
x}{(\log(nx))^{k+1}}}
\nonumber \\
& = & \int_{0}^{1-\frac{\log 2}{\log n}}
\frac{t^k}{(1-t)^{k+1}}e^{-t\log n} \ud t \nonumber \\
& = & \int_{0}^{1/2} \frac{t^k}{(1-t)^{k+1}}e^{-t\log n} \ud t +
\int_{1/2}^{1-\frac{\log 2}{\log n}}
\frac{t^k}{(1-t)^{k+1}}e^{-t\log n} \ud t.
\end{eqnarray}

Now we upper bound the two terms in the last line of eq.
(\ref{int2_eq1}) separately. For the first term we have
\begin{eqnarray}\label{int2_eq2}
& & \int_{0}^{1/2} \frac{t^k}{(1-t)^{k+1}}e^{-t\log n} \ud t
\le 2^{k+1} \int_{0}^{1/2} t^k e^{-t\log n} \ud t \nonumber \\
& \le & \frac{2^{k+1}}{(\log n)^{k+1}} \int_0^{\infty} \tau ^k
e^{-\tau} \ud \tau
\le \frac{2^{k+1} \Gamma(k+1)}{(\log n)^{k+1}} \nonumber \\
& = & O \Bigg(\frac{1}{(\log n)^{k+1}} \Bigg),
\end{eqnarray}
where $\Gamma$ is the gamma function, and the last inequality is due
to that $k$ is a fixed integer.

For the second term we have
\begin{eqnarray}\label{int2_eq3}
& & \int_{1/2}^{1-\frac{\log 2}{\log n}}
\frac{t^k}{(1-t)^{k+1}}e^{-t\log n} \ud t \le \Bigg(\frac{\log
n}{\log 2} \Bigg)^{k+1} \int_{1/2}^{1} e^{-t \log n} \ud t \nonumber
\\ &\le & \frac{1}{2} \cdot \frac{1}{\sqrt{n}} \cdot \Bigg(\frac{\log
n}{\log 2} \Bigg)^{k+1}  = \tilde{O}\bigg( \frac{1}{\sqrt{n}}\bigg),
\end{eqnarray}
where in $\tilde O$ we hide the $\mathrm{polylog}(n)$ terms.

Combining (\ref{int2_eq2}) and (\ref{int2_eq3}) we complete the
proof.

\end{proof}

\begin{proof} \textbf{of Claim \ref{prop_sep}.}

We only need to show that for any $f \in L^2[0,1]$, if
\begin{equation}\label{sep_eq_1}
\int_0^1 f(x) \log^k x \ud x =0,~~~~~k=0,1,\ldots
\end{equation}
then $f=0$ a.e. on $[0,1]$.

Let $t=-\log x$, then eq.(\ref{sep_eq_1}) becomes
\begin{equation*}
\int_0^{\infty} f(e^{-t})t^k e^{-t} \ud t =0,~~~~~k=0,1,\ldots
\end{equation*}
Note that Laguerre polynomials form a complete basis of
$L^2[0,\infty)$
(cf. \citep{orthogonalfunctions}, p.349) , thus $\{t^k\}_{k \ge 0}$
is complete in $L^2[0, \infty)$ with respect to measure $e^{-t}$.
The claim follows.
\end{proof}

\begin{proof} \textbf{of Claim \ref{Claim:for_comparable_lemma}.}
\begin{eqnarray*}
& &\Big|\sum_{r=1}^n{\overline{y}^f(1-r/n)D(r)} - \tilde N^f_D(n) \Big| \\
&=&\Big|\sum_{r=1}^n{\overline{y}^f(1-r/n)D(r)} - \int_1^n{\overline{y}^f(1-s/n)D(s)\ud s}\Big|\\
&=& \Big|\sum_{r=1}^{n-1}{\int_{r}^{r+1}{\left(\overline{y}^f(1-r/n)D(r) - \overline{y}^f(1-s/n)D(s)\right)\ud s}}\Big| + \overline{y}^f(0)D(n)\\
&\leq&
\Big|\sum_{r=1}^{n-1}{\int_r^{r+1}{\overline{y}^f(1-s/n)(D(r)-D(s))\ud
s}}\Big| \\ & &+
\sum_{r=1}^{n-1}{\int_r^{r+1}{\Big|\overline{y}^f(1-r/n)-\overline{y}^f(1-s/n)\Big|D(r)\ud
s}} + \overline{y}^f(0)D(n)\\
&\leq& \sum_{r=1}^{n-1}{\int_r^{r+1}{\Big|D(r)-D(s)\Big|\ud s}} + Cn^{-\alpha/3}\sum_{r=1}^{n-1}{D(r)} + D(n)\\
&\leq& \sum_{r=1}^{n-1}{|D'(r)|} + Cn^{-\alpha/3}F(n) + D(n)\\
&\leq& Cn^{-\alpha/3}F(n) + |D'(1)| + \sum_{r=2}^n |D'(r)| + D(n) \\
&\leq& Cn^{-\alpha/3}F(n) + |D'(1)| + D(1) - D(n) + D(n)\\
& = & Cn^{-\alpha/3}F(n) + |D'(1)| + D(1).
\end{eqnarray*}
Note that the sixth and the seventh line are both because $|D'(r)|$
is monotone decreasing; and second line from bottom is because
$D(r)$ is monotone decreasing.
\end{proof}

\section{Proof of the Convergence
Theorems}\label{section:proof_feasible_discount}

In this section we give the proof of the theorems considering
convergence of NDCG with various discount and cut-off.

First we give the proof of Theorem
\ref{theorem:converge_2_1}, i.e., the standard NDCG converges to $1$ almost surely for every ranking function.\\

\begin{proof}\textbf{of Theorem \ref{theorem:converge_2_1}.}
For notational simplicity we only prove for the case $\mathcal
Y=\{0,1\}$. Generalization is straightforward. Recall that $S_n =
\{(x_1,y_1),\cdots,(x_n,y_n)\}$ consists of $n$ i.i.d.
instance-label pairs drawn from an underlying distribution $P_{XY}$.
Let $p = \Pr(Y=1)$. Also let $l=\sum_{i=1}^{n}y_i$. If $p=0$, the theorem trivially holds. Suppose $p > 0$, by Chernoff
bound we have
\begin{equation*}
\Pr\left(\left|\frac{l}{n}-p\right|> n^{-1/3}  \right) \le
2e^{-2n^{1/3}}.
\end{equation*}
For fixed $n$, conditioned on the event that
$\left|\frac{l}{n}-p\right| \le n^{-1/3}$, by the definition of
NDCG, it is easy to see that
\begin{eqnarray}\label{NDCG_as_converge_1}
\mathrm{NDCG}_D(f,S_n) & = & \frac{\sum_{r=1}^n y^f_{(r)}
\frac{1}{\log (1+r)}}{\sum_{r=1}^l\frac{1}{\log (1+r)}} \nonumber \\
& \ge & \frac{\sum_{r=n-l+1}^n
\frac{1}{\log (1+r)}}{\sum_{r=1}^l\frac{1}{\log (1+r)}} \nonumber \\
& \ge &
\frac{\li(n+1)-\li(n(1-p+n^{-1/3})+1)}{\li(n(p+n^{-1/3})+1)}-o(1) \nonumber \\
& \ge & 1 - o(1),
\end{eqnarray}
where $ \li(t) = \int_2^t \frac{\ud \tau}{\log \tau}$ is the offset
logarithmic integral function; and the last step in
eq.(\ref{NDCG_as_converge_1}) is due to the well-known fact that
$\li(t) \sim \frac{t}{\log t}$. Thus for any $\epsilon > 0$, and for
any sufficiently large $n$, conditioned on the event that
$\left|\frac{1}{n}\sum_{i=1}^{n}y_i - p\right| \le n^{-1/3}$, we
have
\begin{equation*}
\mathrm{NDCG}_D(f,S_n) \ge 1-\epsilon.
\end{equation*}
Also recall that $\mathrm{NDCG}_D(f,S_n) \le 1$. We have, for any
$\epsilon > 0$ and every sufficiently large $n$
\begin{equation*}
\Pr\left(\left| \mathrm{NDCG}_D(f,S_n) - 1 \right| \ge \epsilon
\right) \le 2e^{-2n^{1/3}}.
\end{equation*}
Since $\sum_{n\ge 1}2e^{-2n^{1/3}} < \infty$, by Borel-Cantelli
lemma $\mathrm{NDCG}_D(f,S_n)$ converges to $1$ almost surely.
\end{proof}





Next we give details of the other feasible discount functions as
well as the cut-off versions. In particular, we provide proofs of
Theorems \ref{theorem:r^-alpha}, \ref{theorem:r^-1},
\ref{theorem:kon}, \ref{theorem:kcn_NDCG}, \ref{theorem:kcn_poly}.
The proofs of these five theorems are quite similar. We only prove
Theorem \ref{theorem:r^-alpha} to illustrate the ideas. The proof of
the other four theorems require only minor modifications.

The proof of Theorem \ref{theorem:r^-alpha} relies on the following
lemma, which is similar to Lemma \ref{lemma:lemma3_of_comparable}.

\begin{lemma}\label{lemma:mnclose}
Let $D(r)=r^{-\beta}$ for some $\beta \in (0,1)$. Assume that
$p=\Pr(Y=1)>0$. If the ranking function $f$ satisfies that
$\overline{y}^f(s)=\Pr(Y=1|\tilde{f}(X)=s)$ is continuous, then for
every $\epsilon>0$ the following inequality holds for all
sufficiently large $n$:
\begin{equation*}
\Pr\left[\Big|\mathrm{NDCG}_D(f,S_n)-N^f_D(n)\Big| \ge
5p^{-1}\epsilon \right] \le o(1).
\end{equation*}
\end{lemma}

\begin{proof} \textbf{of Theorem \ref{theorem:r^-alpha}.} The theorem follows from Lemma
\ref{lemma:mnclose} and simple calculations of $\lim_{n \rightarrow
\infty} N^f_D(n)$. We omit the details.
\end{proof}

\begin{proof} \textbf{of Lemma \ref{lemma:mnclose}.}
The proof is simple modification of the proof of Lemma
\ref{lemma:lemma3_of_comparable}. Note that the difference of Lemma
\ref{lemma:mnclose} from Lemma \ref{lemma:lemma3_of_comparable} is
that here we do not assume $y^f$ is H\"{o}lder continuous. We only
assume it is continuous.

Next observe that Claim \ref{Claim:for_comparable_lemma} holds for
$D(r)=r^{-\beta}$ ($0<\beta<1$) as well. Because in the proof of
Claim \ref{Claim:for_comparable_lemma}, we only use two properties
of $D(r)$. That is, $D(r)$ is monotone decreasing and $|D'(r)|$ is
monotone decreasing. Clearly $D(r)=r^{-\beta}$ satisfies these
properties. But here $\overline{y}^f(s)$ is merely continuous rather
than H\"{o}lder continuous. Thus we have a modified version of Claim
\ref{Claim:for_comparable_lemma}. That is, for every $\epsilon > 0$,
the following holds for all sufficiently large $n$:
\begin{equation*}
\Big|\sum_{r=1}^n{\overline{y}^f(1-r/n)D(r)} - \tilde N^f_D(n) \Big|
\le \epsilon F(n) + D(1) +|D'(1)|.
\end{equation*}
The rest of the proof are almost identical to Lemma
\ref{lemma:lemma3_of_comparable}. We omit the details.
\end{proof}

Finally, we give the proof of Theorem \ref{prop_unbound}, i.e., if
the discount decays substantially faster than $r^{-1}$, then the
NDCG measure does not converge. Moreover, every pair of ranking
functions are not strictly distinguishable with high probability by
the measure.\\

\begin{proof}\textbf{of Theorem \ref{prop_unbound}.}
For notational simplicity we give a proof for $|\mathcal{Y}|=2$ and
$\mathcal{Y}=\{0,1\}$. It is straightforward to generalize it to
other cases.

In fact, we only need to show that for every ranking function $f$,
there are constants $a, b, c>0$ with $a>b$, such that for all
sufficiently large $n$,
\begin{equation*}
\Pr[\mathrm{NDCG}_D(f, S_n) \ge a] \ge c
\end{equation*}
and
\begin{equation*}
\Pr[\mathrm{NDCG}_D(f, S_n) \le b] \ge c
\end{equation*}
both hold. Once we prove this, by definition the ranking measure
does not converge (in probability). Also, it is clear that for every
pair of ranking functions, there is at least a constant probability
that the ranking measure of the two functions are ``overlap''.
Therefore distinguishability is not possible.

For sufficiently large n, fix any $x_1,\ldots,x_n$. According to the
assumption, the probability that the top-ranked $m$ data all have
label $1$ is at least $(\delta/2)^m$, where $m$ is the minimal
integer such that
\begin{equation*}
\sum_{r=1}^{m}D(r) \ge \frac{2}{3}\sum_{r=1}^{\infty}D(r).
\end{equation*}
Clearly we have
\begin{equation*}
\Pr\Bigg(\mathrm{NDCG}_D(f,S_n) \ge
\frac{2}{3}~\bigg|~x_1,\ldots,x_n\Bigg) \ge (\delta/2)^m.
\end{equation*}

On the other hand, the probability that the top-ranked $m$ elements
all have label $0$ and there are at least $m$ elements in the list
that have label $1$ is at least $(\delta/2)^{2m}$. Note that
\begin{equation*}
\frac{\sum_{r=m+1}^{n}D(r)}{\sum_{r=1}^{m}D(r)} \le \frac{1}{2}.
\end{equation*}
Thus we have
\begin{equation*}
\Pr[\mathrm{NDCG}_D(f,S_n) \le \frac{1}{2}~|~x_1,\ldots,x_n] \ge
(\delta/2)^{2m}.
\end{equation*}
Since $x_1,\ldots,x_n$ are arbitrary, the theorem follows.
\end{proof}

\section{Proof of Distinguishability for NDCG with $r^{-\beta}$ ($\beta \in (0,1)$) Discount}\label{section:proof_comparable_r^beta}

Here we give the proof of Theorem \ref{theorem:comparable_r^beta},
i.e., NDCG with $r^{-\beta}$ ($0<\beta<1$) discount has the power of
distinguishability. \\

\begin{proof}\textbf{of Theorem \ref{theorem:comparable_r^beta}.}
The proof of distinguishability for polynomial discount is much
easier than that of the logarithmic discount, because in the former
case the pseudo-expectation has very simple form. If $f_0$ and $f_1$
satisfy the first condition $\int_0^1 \Delta y(s) (1-s)^{-\beta} \ud
s \neq 0$, then the theorem is trivially true since
$\mathrm{NDCG}(f_0,S_n)$ and $\mathrm{NDCG}(f_1,S_n)$ converge to
different limits. So we only need to prove the theorem assuming that
$\int_0^1 \Delta y(s) (1-s)^{-\beta} \ud s = 0$ and the second
condition holds. The proof is similar to Theorem
\ref{theorem:comparable} by using the pseudo-expectation. We have
the the next two lemmas for discount $D(r) = r^{-\beta}$, $\beta \in
(0,1)$.

\begin{lemma}\label{lemma:separation_r^-beta}
Let $D(r) = r^{-\beta}$, $\beta \in (0,1)$. Suppose that
$\overline{y}^{f_0}(s)$ and $\overline{y}^{f_1}(s)$ are continuous.
Also assume that $\int_0^1 \Delta y(s) (1-s)^{-\beta} \ud s = 0$ and
$\Delta y (1) \neq 0$. Then we have
\begin{equation}
\Big|N_D^{f_0}(n) - N_D^{f_1}(n)\Big| \geq \Big|\frac{\Delta
y(1)}{2p^{1-\beta}}\Big|\cdot n^{-(1-\beta)}.
\end{equation}
\end{lemma}
\begin{proof}
\begin{eqnarray*}
N_D^{f_0}(n) - N_D^{f_1}(n) &=& \frac{n}{F(np)}\int_{1/n}^1{\Delta y(1-s)\cdot (ns)^{-\beta}\ud s}\\
&=& \frac{1-\beta}{p^{1-\beta}}\int_{1/n}^1{\Delta y(1-s)\cdot s^{-\beta}\ud s}\\
&=& -\frac{1-\beta}{p^{1-\beta}}\int_0^{1/n}{\Delta y(1-s)\cdot
s^{-\beta}\ud s}.
\end{eqnarray*}

Since $\Delta y$ is continuous, for any $\delta > 0$ there exists
$\epsilon > 0$ such that for all $x\in [1-\epsilon,1]$, $|\Delta
y(x)-\Delta y(1)|\leq\delta$. Consequently, for sufficiently large
$n$,
\begin{equation*}
\Big|\int_0^{1/n}{\Delta y(1-s)\cdot s^{-\beta}\ud s} - \Delta
y(1)\cdot\int_0^{1/n}{s^{-\beta}\ud s}\Big| \leq
\delta\cdot\int_0^{1/n}{s^{-\beta}\ud s}.
\end{equation*}
Let $\delta = \Delta y(1)/2$, we then have
\begin{equation*}
\Big|\int_0^{1/n}{\Delta y(1-s)\cdot s^{-\beta}\ud s} - \frac{\Delta
y(1)}{1-\beta}\cdot n^{-(1-\beta)}\Big| \leq \frac{\Delta
y(1)}{2(1-\beta)}\cdot n^{-(1-\beta)}.
\end{equation*}
The lemma follows.
\end{proof}

\begin{lemma}\label{lemma:concentration_r^-beta}
Let $D(r) = r^{-\beta}$, $\beta \in (0,1)$. Assume that
$p=\Pr(Y=1)>0$. If the ranking function $f$ satisfies that
$\overline{y}^f(s)=\Pr(Y=1|\tilde{f}(X)=s)$ is H\"{o}lder continuous
with constants $\alpha>0$ and $C>0$ That is,
$|\overline{y}^f(s)-\overline{y}^f(s')| \le C|s-s'|^{\alpha}$ for
all $s,s' \in [0,1]$. Then
\begin{equation*}{
\Pr\bigg[\Big|\mathrm{NDCG}_D(f,S_n)-N^f_D(n)\Big| \ge
5Cp^{-1}n^{-\min(\alpha/3,1)} \bigg] \le
O\left(e^{-n^{(1-\beta)/3}}\right).}
\end{equation*}

\end{lemma}
\begin{proof}
The proof is almost the same as the proof of Lemma
\ref{lemma:lemma3_of_comparable}
\end{proof}

The theorem follows immediately from Lemma
\ref{lemma:separation_r^-beta} and Lemma
\ref{lemma:concentration_r^-beta}.
\end{proof}



\end{document}